\documentclass{article}

    \usepackage[nonatbib,final]{neurips_2023}

\usepackage[utf8]{inputenc} 
\usepackage[T1]{fontenc}    
\usepackage{url}            
\usepackage{booktabs}       

\usepackage{nicefrac}       
\usepackage{microtype}      
\usepackage{todonotes}
\usepackage{amsmath}
\usepackage{amsthm}
\usepackage{amssymb}
\usepackage{amsmath}
\usepackage{amsfonts}
\usepackage{graphicx}
\usepackage{caption}
\usepackage{subcaption}
\usepackage{enumitem}
\usepackage{algorithm}
\theoremstyle{definition}
\newtheorem{theorem}{Theorem}[section]

\theoremstyle{definition}
\newtheorem{definition}{Definition}[section]
\theoremstyle{remark}

\usepackage{booktabs}
\usepackage{multirow}
\usepackage{xcolor}
\usepackage[numbers,sort&compress]{natbib}
\usepackage[colorlinks]{hyperref}
\definecolor{GGreen}{RGB}{0,128,0}
\hypersetup{
    linkcolor= {blue},
    citecolor={blue}, 
    urlcolor={cyan}
}
\usepackage[nameinlink]{cleveref}
\crefformat{equation}{(#2#1#3)}

\newcommand{\Op}[1]{\operatorname{\mathcal{#1}}}
\newcommand{\A}{\Op{A}}
\newcommand{\X}{\Op{X}}
\newcommand{\x}{x}
\newcommand{\y}{y}
\newcommand{\Y}{\Op{Y}}
\newcommand{\R}{\Op{R}}

\DeclareMathOperator*{\argmin}{arg\,min}

\title{Provably Convergent Data-Driven Convex-Nonconvex Regularization}

\author{%
  Zakhar Shumaylov\\
  University of Cambridge \\
  \texttt{zs334@cam.ac.uk} \\
  \And
  Jeremy Budd \\
  California Institute of Technology\\
  \texttt{jmbudd@caltech.edu} \\
  \And
  Subhadip Mukherjee \\
  IIT Kharagpur \\
  \texttt{subhadipju@gmail.com}
  \And
  Carola-Bibiane Sch\"onlieb \\
  University of Cambridge \\
  \texttt{cbs31@cam.ac.uk} \\
}

\begin{document}

\graphicspath{{media/}}
\maketitle

\begin{abstract}
An emerging new paradigm for solving inverse problems is via the use of deep learning to learn a regularizer from data. This leads to high-quality results, but often at the cost of provable guarantees. In this work, we show how well-posedness and convergent regularization arises within the convex-nonconvex (CNC) framework for inverse problems. We introduce a novel input weakly convex neural network (IWCNN) construction to adapt the method of learned adversarial regularization to the CNC framework. Empirically we show that our method overcomes numerical issues of previous adversarial methods.


  
\end{abstract}

\section{Introduction}
\label{sec:intro}

Inverse problems appear as the main mathematical formulation of a number of scientific applications, including numerous problems in medical imaging. In inverse problems, one seeks to estimate an unknown parameter $x^* \in \X$ from a transformed and noisy measurement
\begin{equation}
    y^{\delta} = \mathcal{A}x^*+ e\in\Y.
    \label{inv_prob_data}
\end{equation}
Here, $ \mathcal{A}:\X\rightarrow \Y$ is the \emph{forward operator}, assumed to be linear and bounded, e.g. representing imaging physics, and $e\in \Y$, with $\left\|e\right\|_{\Y}\leq \delta$, describes measurement noise. However, \cref{inv_prob_data} is frequently \emph{ill-posed}, i.e. the inverse may not exist, be unique, or be continuous in the measurement $y^{\delta}$. 

Traditionally, variational approaches were used to overcome this ill-posedness by hand-crafting a \emph{regularizer} that aims to incorporate prior information about $x^*$ and they are built on top of a rigorous function-analytic foundation. A number of approaches have appeared in recent years which propose to learn a (deep)  regularizer directly from data, see e.g. \cite{elad_ksvd1,venkat_pnp_6737048,romano2017RED,red_schniter,pmlr-v97-ryu19a,gs_denoiser_hurault_2021,kamilov2023plug,chan2016plug,ar_nips,nett_paper, ulyanov2018deepImagePrior,kobler2020total,peng2019auto,uar_neurips2021} and see~\cite{data_driven_inv_prob,dimakis_2022} for overviews. These techniques are often able to achieve high-quality reconstructions, but (unlike traditional variational approaches) often lack provable properties.

The methods described above typically only approach regularization in the sense of model space regularization. This work will explicitly utilize both data and model space regularization (see \cite{fomel1997model}). We will employ the technique of learning \emph{adversarial regularizers} (ARs). This was first proposed in \cite{ar_nips} and has since seen extensions including multi-step regularizers \cite{milne2022new}, latent optimized AR \cite{wang2023convex}, and in \cite{acr_arxiv} deep \emph{input convex neural networks} (ICNNs) (see \cite{amos2017input}) were used to learn an adversarial \emph{convex} regularizer (ACR), which allowed for a number of desirable theoretical guarantees to be proved. However, the imposition of convexity is quite restrictive, and it has been observed in the literature that nonconvex regularizers often have better performance; see \cite{leong2022optimal,4663911,pieper2022nonconvex,roth2009fields}.
Of particular interest is the \emph{convex-nonconvex (CNC) framework}, wherein the regularizer is kept nonconvex in a structured enough way to guarantee convexity of the overall objective, for a review of such methods see \cite{lanza2022convex}. In \cite{goujon2023learning} for example this is achieved by using a ridge regularizer and enforcing the overall function to be 1-weakly convex, turning the denoising objective convex. In a similar manner, a proximal operator corresponding to a weakly convex regularizer can be learned, as in \cite{hurault2022proximal}.

In this work, we adapt the AR method to the CNC framework. In particular, we introduce the 
\emph{input weakly convex neural network} (IWCNN) construction, generalizing ICNNs, to impose weak convexity on regularizers. This 
is motivated by a desire to retain the provable guarantees and desirable optimization properties of ACRs, but also to exploit the advantages of nonconvex regularization. In this formulation we are able to show well-posedness and convergent regularization in the sense of stationary points of the regularizer, using existing results from the literature. We then use our IWCNNs to learn an \emph{adversarial convex-nonconvex regularizer} (ACNCR) and show that this is more versatile than the AR or ACR in computed tomography (CT) experiments.
\section{Background and problem formulation}
\label{sec:background}
In the function-analytic formulation of inverse problems, the unknown parameter $x^*$ is modeled as deterministic and one approximates it from $y^{\delta}$ by solving a variational reconstruction problem:
\begin{equation}
    x_{\alpha}(y^\delta)\in\underset{x \in \X}{\argmin}\text{\,\,}J_{\alpha}\left(\x;\y^\delta\right) :=\mathcal{L}_{\Y}\left(y^{\delta},\A x\right)+\alpha \R(x).
    \label{eq:var_recon}
\end{equation}
Here, $\X$ and $\Y$ are normed vector spaces, $\mathcal{L}_{\Y}:\Y \times \Y \rightarrow{\mathbb{R}^{+}}$ measures \emph{data fidelity}, and the \emph{regularizer} $\R:\X \rightarrow{\mathbb{R}}$ penalizes undesirable images. The parameter $\alpha>0$ trades off data fidelity with regularization and is chosen depending on the noise strength $\delta$. Henceforth, we will choose $\mathcal{L}_{\Y}(y_1,y_2):=\|y_1-y_2\|_{\Y}^2$ for convenience.

In practice, \Cref{eq:var_recon} is solved using iterative optimization schemes, and the quality of reconstructions depends heavily on the choice of $\R$. While convexity of $\R$ may be desirable analytically to provide efficient optimization schemes \cite{subho2023} with various guarantees, in practice reconstruction quality is much better for nonconvex regularizers, as discussed in \Cref{sec:intro}. This, however, comes at a cost: finding global minima becomes impossible in general and sometimes even finding stationary points can not be guaranteed. The \emph{convex-nonconvex (CNC) framework} addresses this problem, wherein the regularizer is kept nonconvex in a structured way to guarantee convexity of the overall objective $J_\alpha(\cdot,y)$.
With this goal, we choose the regularizer as a combination of a weakly convex function over the data space and convex over the model/parameter space, i.e. $\R(\x):= \R^{cnc}(\x,\A\x)$ where $ \R^{cnc}(x,y):=\R^{c}(\x) + \R^{wc}(y)$, where $\R^{c}$ is convex and $\R^{wc}$ is weakly convex.

\begin{definition}[$\rho$-weak convexity]
\label{def:wc}Let $\rho>0$. For $U$ a nonempty convex subset of $\X$, a function $f: U \rightarrow \overline{\mathbb{R}}$ is said to be \emph{$\rho$-weakly convex}, if for all $x_1, x_2 \in U$ and $ \lambda \in[0,1]$,
\begin{equation}\label{eq:weakconvex}
f\left(\lambda x_1+(1-\lambda) x_2\right) \leqslant \lambda f\left(x_1\right)+(1-\lambda) f\left(x_2\right)+\rho\lambda(1-\lambda)\left\|x_1-x_2\right\|^2.
\end{equation}
\end{definition}
The class of weakly convex functions is quite broad and it includes all convex functions and all smooth functions with Lipschitz continuous gradients.

\section{Theoretical results}\label{sec:theory}
Under the CNC parametrization of the regularizer, we now provide convergence guarantees in terms of stability, convergent regularization, and fixed point convergence \cite{subho2023}.
\begin{definition}
    Given $y^0\in \Y$, we say that $x^\dagger$ is an \emph{$\R$-minimizing solution} if \begin{equation}
    \x^{\dagger} \in \underset{x \in \X}{\argmin}\text{\,}\R(\x) \text{\,\,subject to\,\,}\A x=y^{0}.
    \label{r_min_sol}
\end{equation}
Note that an $\R$-minimizing solution can be written as  $\x^{\dagger} \in \underset{x \in \X}{\argmin}\text{\,}\R^{cnc}(x,y^0) \text{\,\,s.t.\,\,}\A x=y^{0}$ due to the constraint, implying uniqueness under e.g. strict convexity of $\R^{cnc}$ in the first argument.
\end{definition}
\begin{theorem}
\label{existence_uniqueness_prop} For $\R^{cnc}(x,y) := \R^{c}(x)+\R^{wc}(y)$ proper, lower semi-continuous, $\mu$-strongly convex in the first argument, and $\rho$-weakly convex and bounded in the second argument, we have:
\begin{enumerate}[leftmargin=*]
    \item  \textbf{Weak convexity: }For $\alpha\rho > 1$:  $J_{\alpha}(\cdot,\y)$ is $-\alpha\mu+(\alpha\rho-1)\|\A\|^2$-weakly convex; For $\alpha\rho \leq 1$: $J_{\alpha}(\cdot,\y)$ is $\alpha\mu$ - strongly convex.
    \item \textbf{Existence:} $J_{\alpha}(\x; \y)$ has a minimizer $x_{\alpha}\left( y\right)$ for every $\y$ and $\alpha>0$. Furthermore for $\alpha\rho\leq1$, $x_{\alpha}\left(\y\right)$ is unique.
    \item \textbf{Stability:} Sequences of minimizers of $J_{\alpha}(\cdot, y^\delta)$ are stable with respect to the data $y^\delta$, i.e. if $\lim _{k \rightarrow \infty} \|y^\delta-y_k\|=0$, then every sequence $x_k \in \argmin J_{\alpha}(\cdot,y_k)$ has a subsequence weakly convergent to some minimizer $x_{\alpha}\left(\y^\delta\right)$. 
\end{enumerate}
Typically, finding global minima of nonconvex objectives is infeasible; we will recover stationary points instead. For  $\widetilde{\x}_{\alpha}\left( y\right)\in\{x\in\X\,|\, 0\in\partial_{x}J_{\alpha}(\x; \y)\}$, for $\partial_{x}J_{\alpha}(\x; \y)$ the Clarke subdifferential: 
\begin{enumerate}[leftmargin=*]
\setcounter{enumi}{3}
\item \textbf{Locally convergent regularization:} For $\|\y^{\delta}-\y^0\|\leq\delta$, $\delta\rightarrow 0$, $\alpha(\delta) \rightarrow 0$, and $\frac{\delta}{\alpha(\delta)}\rightarrow 0$, the reconstruction $\widetilde{x}_{\alpha}\left(\y^{\delta}\right)$ converges to the unique $\R$-minimizing solution $\x^{\dagger}$ in \cref{r_min_sol}. In words: in the vanishing noise limit, there exists a regularization parameter selection strategy under which reconstructions converge to the solution of the noiseless operator equation.
\item \textbf{Convergence of sub-gradient updates:} 
Given the subgradient descent method $x_{k+1}=x_k-\eta_k v_k$, with $v_k \in \partial J_{\alpha}(\cdot,\y)\left(x_k\right)$, if $\alpha\rho\leq 1$: there exists a choice of $\eta_k^*$ such that $\x_{k}$ converge to the minimizer $x_{\alpha}(y)$ with respect to the norm on $\X$. 
\end{enumerate}
\end{theorem}
\begin{proof}
Items (1-5) directly follow from \cite{poschl2009overview,acr_arxiv}. 
\end{proof}
\section{Adversarial convex-nonconvex regularization}
\subsection{Adversarial regularization and the ACR}
Within the adversarial regularization approach, $\mathcal{R}_\theta$ is chosen as a neural network. Assume that we have a dataset of samples $(x_i) \in \X$ and $(y_i) \in \Y$ i.i.d. from the distributions of ground truth images $\mathbb{P}_r$ and measurements $\mathbb{P}_Y$, respectively. We are in the setting of weakly supervised learning, i.e. 
these are \emph{not} samples of measurement-ground truth pairs. In order to `compare' the two distributions, we map $\mathbb{P}_Y$ from the space of measurements $\Y$ to the original space $\X$ using some pseudo-inverse $\A^\dagger$. We denote the projected distribution by $\mathbb{P}_n:=(\A^{\dagger})_{\#} \mathbb{P}_Y$, where $\#$ denotes the push-forward of measures. Then $\mathbb{P}_n$ will correspond to the distribution of images with reconstruction artifacts.


Now, $\mathcal{R}_\theta$ is meant to penalize artificial images and promote real images, so we want $\mathcal{R}_\theta$ to be large on $\mathbb{P}_n$ and small on $\mathbb{P}_r$. 
Therefore, \cite{ar_nips} chose the following loss functional to minimize:
\begin{equation}
    \label{eq:AR_loss}
\mathbb{E}_{X \sim \mathbb{P}_r}\left[\R_{\theta}(X)\right]-\mathbb{E}_{X \sim \mathbb{P}_n}\left[\R_{\theta}(X)\right]+\lambda \cdot \mathbb{E}\left[\left(\left\|\nabla_x \R_{\theta}(X)\right\|-1\right)_{+}^2\right].
\end{equation}
Here, the last term is used to enforce the neural network to be 1-Lipschitz with respect to the input.\footnote{This is done in similarity with the Wasserstein GAN loss (see \cite{arjovsky2017wasserstein}), with the expected value in the last term taken over all lines connecting samples in $\mathbb{P}_n$ and $\mathbb{P}_r$.}



The ACR is then defined in \cite{acr_arxiv} to be of the form $\mathcal{R}_{\theta}(x)=\mathcal{R}^\textrm{ICNN}_\theta(x)+\mu\|x\|^2$ where $\mathcal{R}^\textrm{ICNN}_\theta$ is an ICNN \cite{amos2017input}. This is again trained by minimizing \cref{eq:AR_loss}.

\subsection{Input weakly convex neural network (IWCNN)}\label{sec:IWCNN}
Based on the discussion in \Cref{sec:background,sec:theory}, we wish to construct a neural network parameterization that is nonconvex, but weakly convex, with respect to the input, to be used for adversarial regularization. A natural question to ask is whether the original AR is nonconvex and weakly convex. Unfortunately, the answer is negative due to the following result, which we prove in \Cref{app:propproof}.
\begin{theorem}\label{prop:pwlwc}
A piecewise linear continuous function $f:\mathbb{R}^d\to \mathbb{R}$ with a finite number of pieces (e.g., a \texttt{ReLU} or \texttt{leakyReLU} neural net)  is weakly convex if and only if it is convex.
\end{theorem}
Thus, we need to construct a network which is nonconvex but guaranteed to be weakly convex. For this, we make use of the following fact (for a generalisation to Banach spaces, see \Cref{app:Banach}).
\begin{theorem}[Weak convexity of compositions \cite{davis2018subgradient}]\label{thm:wc_comp}
Let $F(x)=h(c(x))$ for $h: \mathbb{R}^m \rightarrow \mathbb{R}$ convex and $L$-Lipschitz, and $c: \mathbb{R}^d \rightarrow \mathbb{R}^m$\: a $C^1$-smooth map with $\beta$-Lipschitz gradient.
Then $F$ is $L \beta$-weakly convex. Furthermore, by the chain rule, 
$\partial F(x)=\nabla c\left(x\right)^\top \partial h\left(c\left(x\right)\right)$. 
\end{theorem}
\begin{definition}[IWCNN]\label{def:IWCNN}
By \Cref{thm:wc_comp}, we can therefore define an IWCNN by 
\[
f^\textrm{IWCNN}_\theta=g^\textrm{ICNN}_{\theta_1}\circ g^\textrm{sm}_{\theta_2},\] where $g^\textrm{ICNN}_{\theta_1}$ is an ICNN and $g^\textrm{sm}_{\theta_2}$ is a neural network with smooth activations. 
\end{definition}
We note that the IWCNN construction is only necessary because of non-smooth activation functions (like \texttt{ReLU} or \texttt{leakyReLU}), which are used in both the AR and ACR. A neural network with all smooth activation functions, e.g. \texttt{sigmoid}, is automatically weakly convex. However, usage of smooth non-linearities has been shown to worsen performance in machine learning models \cite{krizhevsky2017imagenet} and furthermore using smooth activations in the original AR formulation turns out to harm performance.

\subsection{Adversarial convex-nonconvex regularizer (ACNCR)}
With the IWCNN in hand, we can now parametrize the learned regularizer as $\R_\theta^{cnc}$ defined as $\R_\theta^{{cnc}}(\x,\y) := \R_{\theta_1}^{c}(\x) + \R_{\theta_2}^{wc}(\y)$, for $\R_{\theta_1}^{c}$ parameterized in the same way as the ACR, while $\R_{\theta_2}^{wc}$ is parameterized using an IWCNN. Thus, we can view this approach as learning to denoise in both the data and observation domain. Note that this ACNCR is strictly more expressive than the ACR.
Thus, denoting $\mathbb{P}_{Y_r} :=\left(\A\right)_{\#} \mathbb{P}_r$, we train both of the networks in a decoupled way by minimizing:
\begin{equation*}
    \begin{aligned}\label{eq:ACNCR_loss}
\mathbb{E}_{X \sim \mathbb{P}_r}\left[\R_{\theta_1}^{c}(X)\right] &- \mathbb{E}_{X \sim \mathbb{P}_n}\left[\R_{\theta_1}^{c}(X)\right] &+\lambda \cdot \mathbb{E}\left[\left(\left\|\nabla_x  \R_{\theta_1}^{c}(X)\right\|-1\right)_{+}^2\right]&&+ \\
\mathbb{E}_{Y \sim \mathbb{P}_{Y_r}}\left[\R_{\theta_2}^{wc}(Y)\right] &- \mathbb{E}_{Y \sim \mathbb{P}_Y}\left[\R_{\theta_2}^{wc}(Y)\right] &+\lambda \cdot \mathbb{E}\left[\left(\left\|\nabla_y \R_{\theta_2}^{wc}(Y)\right\|-1\right)_{+}^2\right]&.&
\end{aligned}
\end{equation*}
In choosing this objective, we normalized the operator $\A$ to have norm 1. This is also done in experiments to further aid stable training of the networks. 
\section{Computed Tomography (CT) numerical experiments}\label{sec:numerics}
To evaluate our ACNCR, we consider two applications: CT reconstruction with (i) sparse-view and (ii) limited-angle projection. For details on the experimental set-up, see \Cref{app:expsetup}.
Main results are shown in \Cref{sample-table}, with further visual examples in \Cref{app:visual}.

\paragraph{Sparse view CT}As in \cite{ar_nips} performance of AR during reconstruction deteriorates if the network is over-trained and early stopping is not employed. For ACR this does not occur due to reduced expressivity, at a price of reduced performance as seen on \Cref{sample-table}. Akin to ACR, ACNCR overcomes this limitation and performs on par with AR, without over-training and without early stopping thanks to better expressivity. 

\paragraph{Limited view CT} Reconstruction from limited-angle projection data, with no measurement in a specific angular region, is an inverse problem with a severely ill-posed forward operator where the reconstruction performance depends critically on the image prior. One of the main benefits of imposing convexity on the regularizer is the improved performance in the limited-angle setting as compared to AR, wherein, even with early stopping, artifacts arise in the reconstructions. ACNCR overcomes this issue without having to employ early stopping, while also performing on par with ACR and outperforming both model-based approaches and AR as seen from \Cref{sample-table}. 

\section{Conclusion}
In this work, we have shown that a CNC regularizer defined as the sum of a weakly convex function over the data space and a convex function over the model/parameter space exhibits existence of solutions, stability, convergent regularization, and convergence of subgradient descent. Furthermore, through our novel IWCNN construction, we have shown how to learn such a regularizer adversarially. In CT experiments, this  ACNCR is observed to better adapt to both sparse and limited-angle settings, showing that it is more versatile with respect to ill-posedness than the AR and ACR approaches.
\newpage
\appendix

\section{Visualization of experimental results}
\label{app:visual}
\begin{table}[h]
  \caption{Average PSNR and SSIM over test data in CT experiments.}
  \label{sample-table}
  \centering
  \begin{tabular}{l c c r c c r}
    \toprule
    & \multicolumn{3}{c}{Limited} & \multicolumn{3}{c}{Sparse}                   \\
    \cmidrule(r){2-4} \cmidrule(r){5-7}
    Methods  & PSNR (dB)       & SSIM        & \# param. & PSNR (dB)       & SSIM        & \# param.\\
    \midrule
    FBP & 17.1949  & 0.1852 & 1 & 21.0157  & 0.1877 & 1 \\
    TV & \textbf{25.6778}  & \textbf{0.7934} & 1 & \textbf{31.7619}  & \textbf{0.8883} & 1 \\
    \midrule
    LPD &  28.9480 & \textbf{0.8394} & 127\,370 & \textbf{37.4868} & 0.9217 & 700\,180 \\
    FBP + U-Net &  \textbf{29.1103} & 0.8067 & 14\,787\,777 &  37.1075 &  \textbf{0.9265} & 14\,787\,777 \\
    \midrule
    AR & 23.6475 & 0.6257 & 133\,792 & 36.4079  & \textbf{0.9101} & 33\,952\,481 \\
    ACR &  26.4459 & \textbf{0.8184} & 34\,897 & 34.5844 & 0.8765 & 9\,448 \\
    ACNCR & \textbf{26.5420} & 0.8161 & 1\,085\,448 & 35.6476  &  0.9094 & 1\,085\,448 \\
    \bottomrule
  \end{tabular}
\end{table}
\begin{figure*}[h]
	\centering
	\begin{subfigure}[b]{1.35in}
    \caption{Ground Truth}
    \includegraphics[width=1.35in]{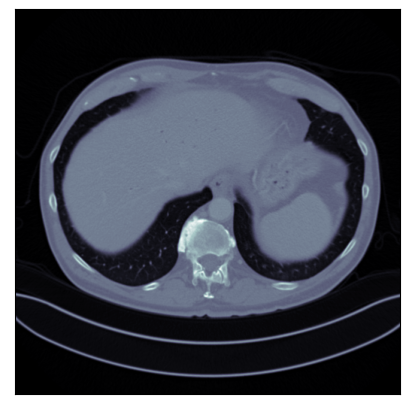}
    \caption*{\tiny{PSNR:$\infty$, SSIM:1.0000}}
    \end{subfigure}
    \begin{subfigure}[b]{1.35in}
    \caption{FBP}
    \includegraphics[width=1.35in]{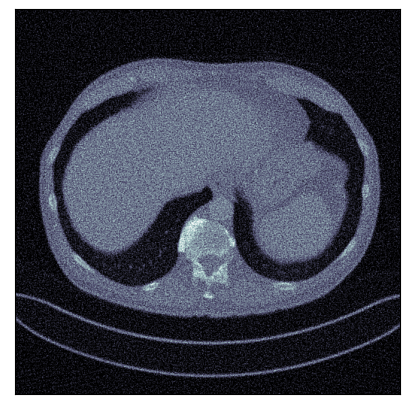}
    \caption*{\tiny{PSNR:20.9888, SSIM:0.2131}}
    \end{subfigure}%
    \begin{subfigure}[b]{1.35in}
    \caption{TV}
    \includegraphics[width=1.35in]{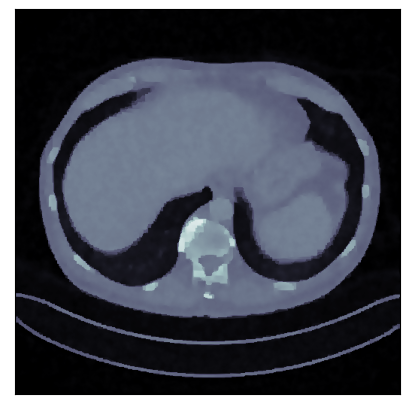}
    \caption*{\tiny{PSNR:29.9556, SSIM:0.8733}}
    \end{subfigure}%
    \begin{subfigure}[b]{1.35in}
    \caption{AR}
    \includegraphics[width=1.35in]{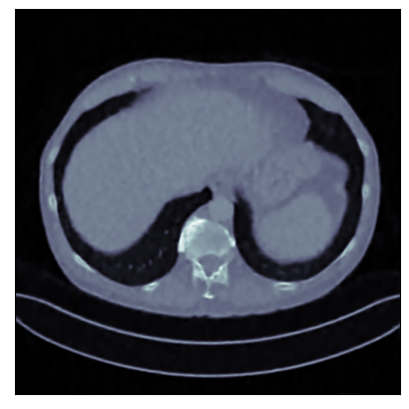}
    \caption*{\tiny{PSNR:35.8974, SSIM:0.9125}}
    \end{subfigure}%
    
    \begin{subfigure}[b]{1.35in}
    \caption{ACNCR}
    \includegraphics[width=1.35in]{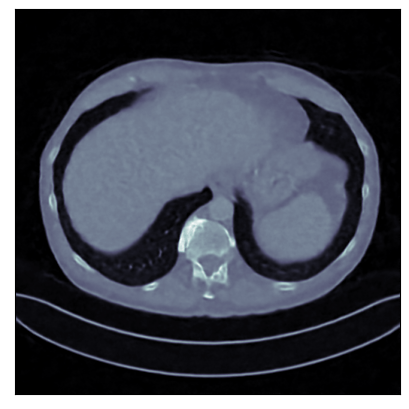}
    \caption*{\tiny{PSNR:36.9551 dB, SSIM:0.9155}}
    \end{subfigure}%
    \begin{subfigure}[b]{1.35in}
    \caption{ACR}
    \includegraphics[width=1.35in]{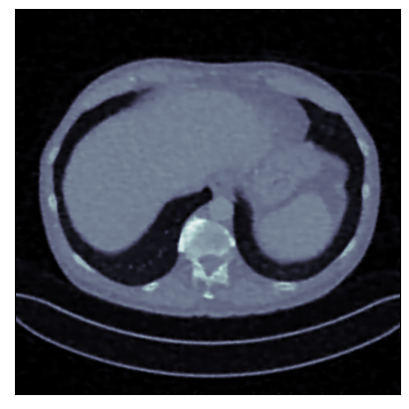}
    \caption*{\tiny{PSNR:34.9924, SSIM:0.8912}}
    \end{subfigure}%
    \begin{subfigure}[b]{1.35in}
    \caption{LPD}
    \includegraphics[width=1.35in]{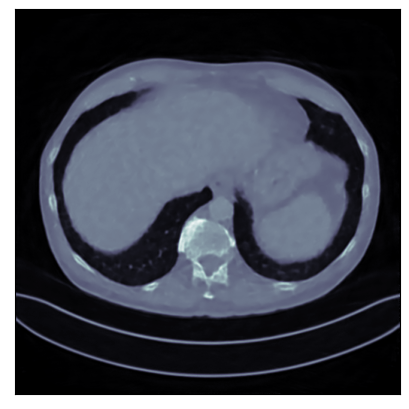}
    \caption*{\tiny{PSNR:36.4687, SSIM:0.9172}}
    \end{subfigure}%
    \begin{subfigure}[b]{1.35in}
    \caption{U-NET}
    \includegraphics[width=1.35in]{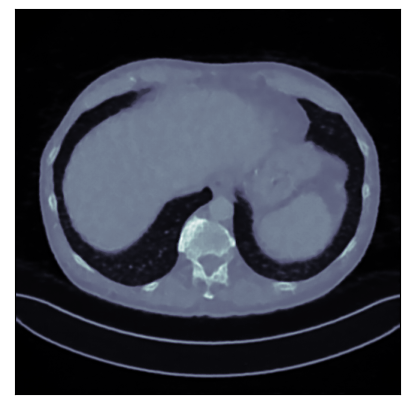}
    \caption*{\tiny{PSNR:36.5235, SSIM:0.9269}}
    \end{subfigure}%
 	\caption{\small{Reconstructed images obtained using different methods, along with the associated PSNR and SSIM, for sparse view CT. In this case, the ACNCR is shown to outperform the AR in terms of SSIM and PSNR.}}
	\label{fig:image_figure}
\end{figure*}

\section{Experimental set-up}\label{app:expsetup}
For experiments, human abdominal CT scans for 10 patients in the Mayo Clinic low-dose CT grand challenge dataset \cite{mayo_ct_challenge} are used. We simulate the projection data in ODL \cite{odl}.  Our training dataset for the CT experiments consists of 2250 2D slices of size $512 \times 512$ corresponding to nine patients, and the slices extracted from the remaining one patient are used for evaluation. 

Our ACNCR method is compared with two model-based techniques: filtered back-projection (FBP) and total variation (TV) regularization; two supervised data-driven methods: the learned primal-dual (LPD) method \cite{lpd_tmi} and UNet-based post-processing of FBP \cite{postprocessing_cnn}; and two adversarial regularization approaches: the AR and the ACR. In comparing with these we illustrate the trade-off in levels of constraints versus stability and performance. The peak signal-to-noise ratio (PSNR) and structural similarity index (SSIM) \cite{ssim_paper_2004} are used as quality metrics. For fairness, we report the highest PSNR achieved by all methods during reconstruction. Results are displayed in \Cref{app:visual}.

LPD is trained on pairs of target images and projection data, whereas the U-net post-processor is trained on pairs of true images and the corresponding FBP. AR, ACR and AWCR, in contrast, require ground-truth and FBP images drawn from their marginal distributions (and hence not necessarily paired). The hyperparameters $\lambda$ and $\rho_0$ are chosen in the same as done in \cite{acr_arxiv}.

For both CT experiments, projection data is simulated using a parallel-beam acquisition geometry with 350 angles and 700 rays/angle, using additive Gaussian noise with $\sigma=3.2$. The pseudoinverse reconstruction is taken to be images obtained using FBP. For limited angle experiments, data is simulated with a missing angular wedge of 60$^{\circ}$. The native \texttt{odl} power method is used to approximate the norm of the operator. For all three adversarial regularisation methods a fixed number of steps of accelerated gradient descent is performed and for consistency best PSNR value images are reported in \Cref{sample-table}.

For the IWCNN \Cref{sec:IWCNN} architecture, the usual ICNN \cite{amos2017input} architecture using \texttt{leakyReLU} activations is used for $g^\textrm{ICNN}_{\theta_1}$, while the smooth network $g^\textrm{sm}_{\theta_2}$ is parameterised a deep convolutional network with \texttt{SiLU} activations and 5 layers. The \textit{RMSprop} optimizer (following \cite{ar_nips}) with a learning rate of $1\times 10^{-4}$ is used for training.



\section{Proof of \Cref{prop:pwlwc}}\label{app:propproof}
\begin{proof}
    If $f$ is convex, then by definition it is weakly convex. 
    
    Suppose that $f$ is weakly convex. As $f$ is convex if and only if $f|_\ell$ is convex for all lines $\ell\subseteq\mathbb{R}^d$, and if $f$ is weakly convex then so is $f|_\ell$ for all lines $\ell\subseteq\mathbb{R}^d$, it suffices to prove the theorem for $d = 1$. 

    Since $f:\mathbb{R}\to\mathbb{R}$ is piecewise linear with finitely many pieces, there exist $(a_i)_{i=1}^{n+1}$, $(b_i)_{i=1}^{n+1}$, and $(x_i)_{i=1}^n$ such that the $x_i$ are monotonically increasing and, defining $x_0=-\infty$ and $x_{n+1}=\infty$,
    \[
    f(x) = \begin{cases}
        a_i x + b_i, & x \in (x_{i-1},x_i], \: i \in \{1,2,...,n+1\}  .
    \end{cases}
    \]
    Furthermore, by continuity, for all $i\in  \{1,...,n\}$, $a_ix_i+b_i = a_{i+1}x_i + b_{i+1}$. 

    Since $f$ is weakly convex, there exists $\rho > 0$ such that, by \cref{eq:weakconvex}, for all $\varepsilon > 0$ and $i\in  \{1,...,n\}$,
    \[
    f(x_i) = f\left(\frac12(x_i-\varepsilon) + \frac12(x_i+\varepsilon) \right) \leq \frac12 f(x_i - \varepsilon) + \frac12 f(x_i + \varepsilon) + \frac12 \times \frac12 \rho (2\varepsilon)^2.
    \]
    For $\varepsilon < \min\{x_i-x_{i-1}, x_{i+1}-x_i\}$, this becomes
    \[
    a_ix_i + b_i  \leq \frac12 (a_i(x_i - \varepsilon) + b_i )+ \frac12 (a_{i+1}(x_i + \varepsilon) + b_{i+1} ) + \rho \varepsilon^2
    \]
    which simplifies to (using that $a_ix_i+b_i = a_{i+1}x_i + b_{i+1}$)
     \begin{align*}
    0 \leq \frac12 \varepsilon(a_{i+1} -a_i ) + \rho \varepsilon^2, &&
    \text{and therefore} && 
    a_{i+1} -a_i \geq -2 \rho \varepsilon,
    \end{align*}
    for all sufficiently small $\varepsilon > 0$. Hence, $a_{i+1} \geq a_i$ for all $i\in \{1,..,n\}$. 
    
    We make the following claim: 
    \begin{equation}\label{eq:fmaxclaim}
        f(x) = \max_{j \in \{1,...,n+1\}} a_j x + b_j =: f_j(x).
    \end{equation}
     To prove \cref{eq:fmaxclaim}, let $x \in (x_{i-1},x_i]$, and hence $f(x)=f_i(x)$. Note that for all $j$, $f_j(x_j) = f_{j+1}(x_j)$ and that for all $j$ and $y$, $f'_j(y) = a_j \leq a_{j+1} =  f'_{j+1}(y)$. Hence: 
    \begin{itemize}[leftmargin=*]
        \item For all $j$ and $y\geq x_{j-1}$, $f_{j}(y) \geq f_{j-1}(y)$.
        \item For all $j$ and $y\leq x_j$, $f_j(y) \geq f_{j+1}(y) $.
    \end{itemize}
    It follows that: 
\begin{itemize}[leftmargin=*]
        \item For all $j < i$ and $y\geq x_{i-1}$, $f_{i}(y) \geq f_j(y)$.
        \item For all $j > i$ and  $y\leq x_{i}$, $f_i(y) \geq f_{j}(y) $.
    \end{itemize}
    Since $y = x$ satisfies both conditions, we have that for all $j \neq i$, $f_i(x) \geq f_j(x)$, as desired. Finally, from \cref{eq:fmaxclaim} it immediately follows that $f$ is convex, as it is the pointwise maximum of affine (and therefore convex) functions.
\end{proof}

\section{Banach composition weak convexity}
\label{app:Banach}
\begin{theorem}[Special case of  \cite{thibault2021unilateral} Proposition 10.21(c)]Let $U$ open and $V$ be nonempty convex sets of two normed spaces $X$ and $Y$ respectively, let $f: V \rightarrow \bar{\mathbb{R}}$ be finite, convex and $K_f$-Lipschitz continuous on $V$.
Let $F$ : $U \rightarrow Y$ be a differentiable mapping with $F(U) \subset V$ and such that $D F$ is uniformly continuous on $U$ with a linear modulus of continuity (i.e. Lipschitz continuous ) with coefficient $K_F$ on $U$
, then $f \circ F$ is $K_f K_F$-weakly convex.
\end{theorem}

\bibliographystyle{plainnat}  
 \bibliography{refs}

\end{document}